\documentclass{article}
\usepackage{spconf}
\ninept 

\usepackage[utf8]{inputenc} 
\usepackage[T1]{fontenc}    
\usepackage{hyperref}       
\usepackage{url}            
\usepackage{booktabs}       
\usepackage{amsfonts}       
\usepackage{nicefrac}       
\usepackage{microtype}      
\usepackage{siunitx}
\usepackage{enumerate}
\usepackage{float}
\usepackage{caption}
\usepackage{soul}

\usepackage{lipsum}

\hypersetup{allcolors=blue,colorlinks=true}

\usepackage{graphicx}
\usepackage{amsmath}
\usepackage{amsthm}
\usepackage{amssymb}
\usepackage{amsfonts}
\usepackage{mathrsfs}
\usepackage{multirow}
\usepackage{algorithm}
\usepackage{subfigure}
\usepackage{thm-restate}
\usepackage{mathtools}
\usepackage{comment}
\usepackage{multicol}
\usepackage{bm}
\newcommand{\R}{\mathbb{R}}
\newcommand{\C}{\mathbb{C}}
\newcommand{\B}{\mathbb{B}}
\newcommand{\bA}{\bm{A}}
\newcommand{\bB}{\bm{B}}

\newcommand{\bD}{\bm{D}}

\newcommand{\bI}{\bm{I}}

\newcommand{\bU}{\bm{U}}
\newcommand{\bV}{\bm{V}}

\newcommand{\bX}{\bm{X}}

\newcommand{\bu}{\bm{u}}

\newcommand{\bw}{\bm{w}}
\newcommand{\bx}{\bm{x}}

\newcommand{\bZ}{\bm{Z}}

\newcommand{\im}{\mathfrak{i}}

\usepackage{bbm}

\DeclareMathOperator*{\argmin}{arg\,min}

\DeclareMathOperator{\prox}{\textsf{prox}}

\usepackage{pifont}

\newtheorem{proposition}{Proposition}

\usepackage[natbib=true,
            bibstyle=ieee,
            citestyle=numeric-comp,
            uniquename=false,
            uniquelist=false,
            maxcitenames=2,
            mincitenames=1,
            mincrossrefs=99]{biblatex}
\renewbibmacro{in:}{\ifentrytype{article}{}{\printtext{\bibstring{in}\intitlepunct}}}
\DefineBibliographyStrings{english}{url = Available at}

\AtEveryBibitem{%
  \clearlist{publisher}%
  \clearfield{pagetotal}%
  \clearfield{series}%
  \clearfield{edition}%
  \clearfield{isbn}%
  \clearname{editor}%
  \clearlist{location}%
}

\title{Automatic Registration and Clustering of Time Series}

\name{Michael Weylandt\sthanks{MW's research is supported by an appointment to the Intelligence Community Postdoctoral Research
Fellowship Program at the University of Florida Informatics Institute, administered by Oak Ridge Institute for Science and Education through an interagency agreement between the U.S. Department of Energy and the Office of the Director of National Intelligence.} and George Michailidis\sthanks{GM's work was supported in part by NSF grants DMS 1830175 and DMS 1821220. }}
\address{University of Florida Informatics Institute\\ University of Florida \\ Gainesville, FL 32611.\\
         \href{mailto:michael.weylandt@ufl.edu}{michael.weylandt@ufl.edu} and \href{mailto:gmichail@ufl.edu}{gmichail@ufl.edu}}

\begin{document}
\begin{refsection}
\maketitle

\begin{abstract}
Clustering of time series data exhibits a number of challenges not present in other settings, notably the problem of registration (alignment) of observed signals. Typical approaches include pre-registration to a user-specified template or time warping approaches which attempt to optimally align series with a minimum of distortion. For many signals obtained from recording or sensing devices, these methods may be unsuitable as a template signal is not available for pre-registration, while the  distortion of warping approaches may obscure meaningful temporal information. We propose a new method for automatic time series alignment within a \emph{clustering problem}. Our approach, Temporal Registration using Optimal Unitary Transformations (TROUT), is based on a novel dissimilarity measure between time series that is easy to compute and automatically identifies optimal alignment between pairs of time series. By embedding our new measure in a optimization formulation, we retain well-known advantages of computational and statistical performance. We provide an efficient algorithm for TROUT-based clustering and demonstrate its superior performance over a range of competitors.

\noindent \keywords{time series clustering, convex clustering, time series registration, majorization-minorization algorithms}
\end{abstract}

\section{Introduction}
\subsection{Convex Clustering}
The previous two decades have witnessed a flurry of activity in statistical machine learning and statistical signal processing, as classical methods have been reinterpreted as convex penalized estimation problems and powerful new theory has extended this methodology to the high-dimensional (under-sampled) setting. While the most notable advances have been in supervised methods such as regression and signal reconstruction, the convexity revolution has also yielded powerful new techniques for unsupervised learning. Notably, a convex formulation of clustering \citep{Lindsten:2011,Hocking:2011} has received significant attention as it allows for theoretical and computational guarantees not easily obtained for classical clustering approaches \citep{Zhu:2014,Tan:2015,Radchenko:2017}. Traditional convex clustering combines a squared Frobenius norm data fidelity term, which keeps the estimated cluster centroids near the original data, with a convex fusion penalty, which shrinks the estimated centroids together, yielding the following estimator: 
\[\hat{\bU} = \argmin_{\bU \in \R^{n \times p}} \frac{1}{2}\|\bX - \bU\|_F^2 + \lambda \sum_{\substack{i, j = 1 \\ i < j}}^n w_{ij} \|\bU_{i\cdot} - \bU_{j\cdot}\|_q\]
where $\bX \in \R^{n \times p}$ is the data matrix, $\hat{\bU} \in \R^{n \times p}$ is the matrix of estimated centroids, $\lambda > 0$ is a regularization parameter controlling the overall degree of fusion, and $\{w_{ij}\}$ are optional problem-specific non-negative fusion weights. Observations $i$ and $j$ are said to belong to the same cluster if $\hat{\bU}_{i\cdot} = \hat{\bU}_{j\cdot}$. In this paper, we only consider the convex $\ell_q$ fusion penalty ($q = 1, 2, \infty$), though some research suggests non-convex penalties also perform well \citep{Marchetti:2014}.

Focusing our attention on the data fidelity term, we note that any distance function can be used in lieu of the Frobenius norm without sacrificing the advantages of the convex formulation. \citet{Liu:2019} replace the Frobenius loss with a Huber loss to make the clustering solution robust to outliers, while \citet{Wang:2019-GECCO} propose a framework for incorporating arbitrary Bregman divergences and exponential family-type distances in the convex clustering framework. \citet{Sui:2018} use a Mahalanobis-type distance and embed convex clustering in a metric learning scheme. Extending convex clustering to non-vector-valued data, \citet{Park:2018} use an earth-mover's distance to cluster histogram-valued observations arising in genomic sequencing. Building on this line of research, we propose a novel automatic spectral-registration dissimilarity suitable for use with temporal data and investigate its use in the clustering context.

\begin{figure*}
\includegraphics[width=\textwidth]{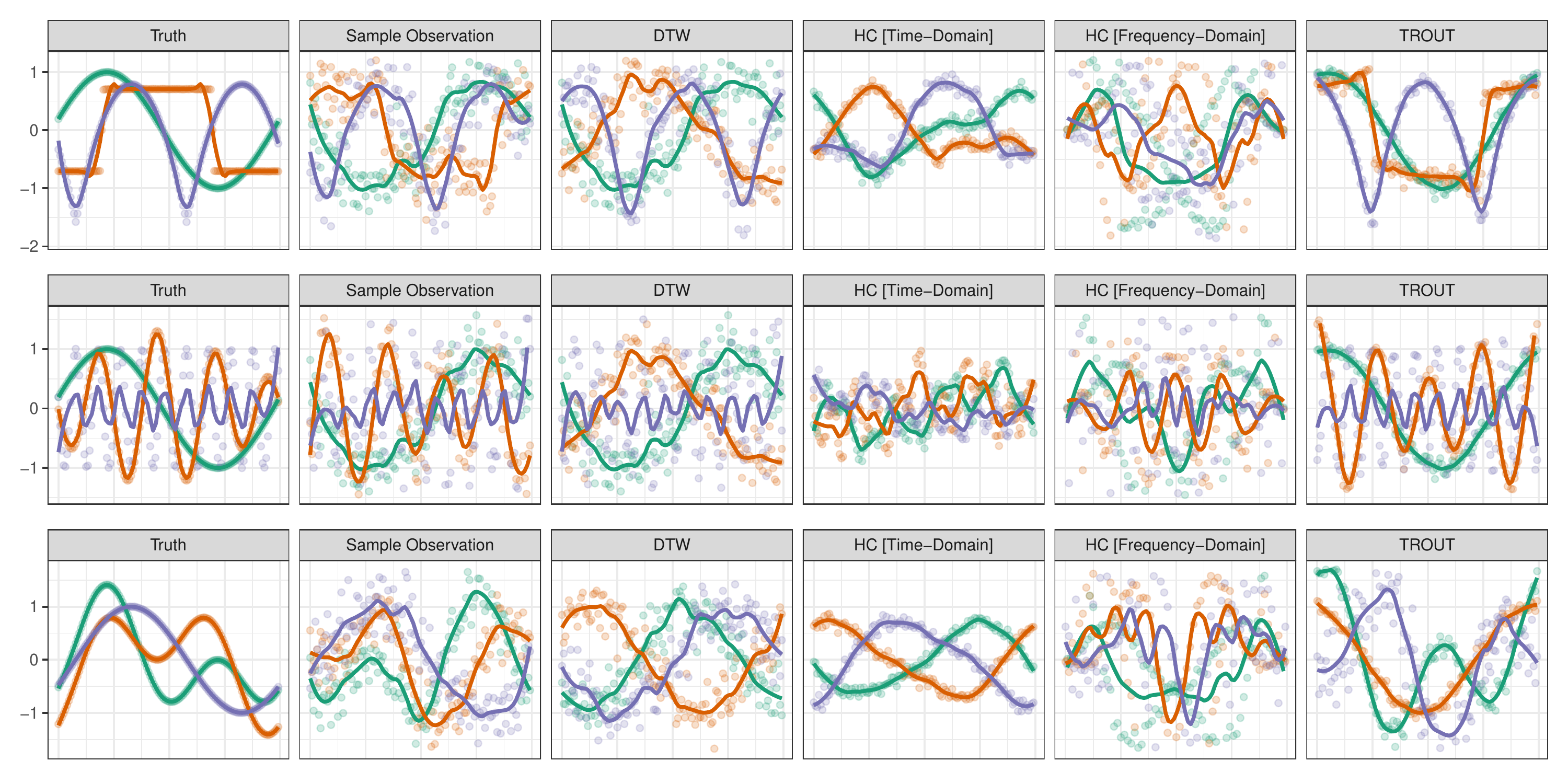}
\caption{Simulation Scenarios Used in Section \ref{sec:sims}. The first scenario (top row) consists of three signals which are clearly distinct in the time domain: a sinusoid, a smoothed triangle wave, and a square wave. The second scenario (middle row) consists sinusoids at three different frequencies. Third scenario (bottom row) consists of combinations of pairs of sinusoids with different relative phase.From left to right, the panels depict: the true cluster centroid (true signal); a sample observation, here with $\text{SNR} = 3$ and von Mises($\kappa = 10$) phase noise; centroids from DTW + hierarchical clustering (HC), na\"ive time-domain HC, frequency-domain HC, and TROUT.}
\vspace{-0.15in}
\label{fig:sim_design}
\end{figure*}

\subsection{Time Series Clustering}
Before introducing our proposal, we review some of the particular challenges associated with clustering time series data. Specifically, we focus on what \citet{Aghabozorgi:2015} term ``whole time-series clustering,'' wherein multiple time series are grouped into a small number of distinct clusters; they distinguish this from ``subsequence clustering,'' identifying subsequences within a single long time series, and ``time point clustering,'' the task of clustering individual points in a way which respects their temporal ordering. Most commonly, whole series clustering is performed by identifying a suitable distance measure between series and using that distance measure within a classical clustering method such as $k$-means or hierarchical clustering.

Perhaps the most frequently used time series distance measure is Dynamic Time Warping (DTW), originally proposed by \citet{Sakoe:1971, Sakoe:1978}. DTW uses a dynamic programming approach to optimally insert or remove additional time points to minimize the Euclidean distance between the time-domain representation of two signals. DTW is particularly well-suited for natural signals, such as those arising from social, economic, or consumer-behavior processes, where the ``shape'' of the signal is more meaningful than the exact temporal structure. The flexibility of DTW is, for some applications, a double-edged sword so a maximum distortion constraint of 10\%, following \citeauthor{Sakoe:1978}, is imposed. \citet{Ratanamahatana:2005} argue that this flexibility is often quite detrimental for clustering performance and suggest that a much lower deformation constraint be used in practice. 

Alternatives to DTW include na\"ive (time-domain) metrics, where the temporal structure is ignored, as well as clustering using the Fourier or wavelet coefficients of the time series as input data. Spectral representations have the advantage of limiting the temporal dependence and integrating well with established signal processing techniques. See the review paper by \citet{Serra:2014} for a thorough evaluation of time series metrics in the classification context. 

\section{Automatic Registration of Time Series}

While the flexibility of DTW and related methods may be suitable for some problems, we instead consider an alignment method tailored for \textit{stationary} signals. Specifically, we consider signals where the temporal information \emph{within} a signal is meaningful, but the different signals are not aligned to each other. Considering these signals in the spectral (Fourier) domain, only the relative phase of the different components is meaningful, while the absolute phase is not. We model this situation using ``\emph{phase noise},'' where a random shift is applied uniformly across the signal, corresponding to our (random) observational window. To address phase noise, we propose a new distance measure between which only depends on the relative phase of various components, and not the absolute phase. As a generalization of Euclidean distance, our method also handles observational (white) noise in the usual manner.

Given two univariate time series, we can consider three types of distances based on their spectral representations: i) \emph{phase-sensitive}, corresponding to the standard Euclidean distance between their Fourier coefficients; ii) \emph{phase-oblivious}, corresponding to the Euclidean distance between the magnitudes of their Fourier coefficients (retaining power, but discarding phase); and iii) \emph{phase-adaptive}, corresponding to the Euclidean distance between their \emph{optimally-aligned} Fourier coefficients. Formally, the phase-adaptive spectral distance can be defined as
\begin{equation}
  d_{\text{TROUT}}(\bu, \bx) = \min_{\theta \in [0, 2\pi)} \|\bu - e^{\im \theta} \bx\|_2. \label{eqn:trout}
\end{equation}
Compared to DTW, $d_{\text{TROUT}}$ preserves the temporal structure in the signal, while identifying the optimal alignment between pairs of signals. 

Given two matrices $\bX, \bU \in \C^{n \times p}$, each row of which represents the DFT of a separate series, the TROUT distance is defined as the sum of row-wise distances:
\[d_{\text{TROUT}}(\bU, \bX)^2 = \sum_{i=1}^n d_{\text{TROUT}}(\bU_{i\cdot}, \bX_{i\cdot})^2.\]
Combining this matrix distance with a convex fusion penalty yields our proposed method for simultaneous time series registration and clustering:
\begin{equation}
  \argmin_{\bU \in \C^{n \times p}} \frac{1}{2}d_{\text{TROUT}}(\bU, \bX)^2 + \lambda \sum_{\substack{i, j = 1 \\ i < j}}^n w_{ij} \|\bU_{i\cdot} - \bU_{j\cdot}\|_q. \label{eqn:trout_clust}
\end{equation}
An important aspect of our formulation is that $d_{\text{TROUT}}$ implicitly re-aligns the series at each step. Compared with standard pre-alignment techniques, this can significantly improve performance and interpretability. Specifically, $d_{\text{TROUT}}$ works by aligning each estimated centroid (row of $\bU$) with the observed data. This is preferable to pre-alignment as it avoids the difficult question of selecting a registration target, and aligns a centroid only to that data in its cluster. Similar advantages of dynamic alignment have been observed in the shape analysis context by \citet{Srivastava:2011}, though our approach is simpler to implement than theirs. We note that $d_{\text{TROUT}}$ fails to satisfy the triangle inequality and hence is not a proper distance measure. As such, Problem \eqref{eqn:trout_clust} is not generally convex without further assumptions on $\bX$.

As a generalization of Euclidean convex clustering, TROUT clustering inherits many of the attractive properties of convex clustering including: a continuous solution path, indexed by $\lambda$, which smoothly interpolates from $n$ to $1$ clusters \citep{Hocking:2011}; the ability to quickly identify local optima \citep{Chi:2015}; the ability to construct dendrogram representations of the solution path \citep{Weylandt:2020}; provable statistical consistency \citep{Zhu:2014,Tan:2015,Radchenko:2017}; and a high-degree of robustness to noise \citep{Wang:2019-GECCO,Liu:2019}. While our focus is on the use of $d_{\text{TROUT}}$ in the fusion clustering context, we also note that the TROUT distance, and associated notion of TROUT Fr\'echet mean, can also be used within $K$-means or hierarchical clustering.

\begin{figure*}
\includegraphics[width=\textwidth]{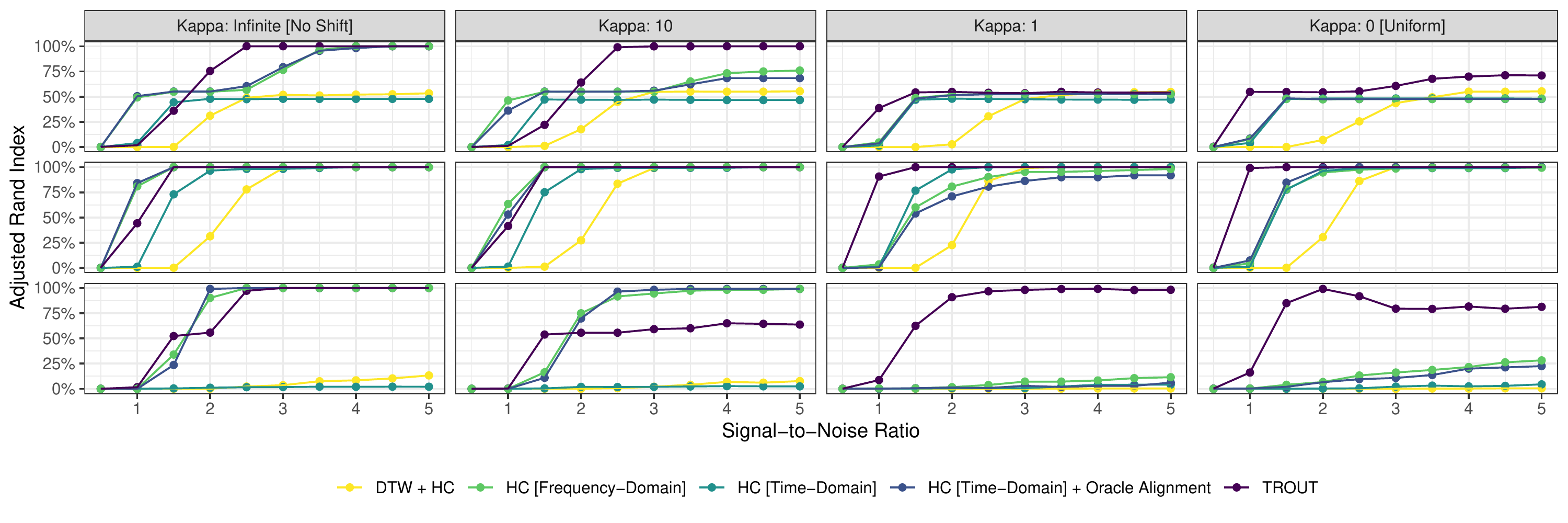}
\caption{Comparitive performance of TROUT clustering against dynamic time warping (DTW), time-domain, and frequency-domain based hierarchical clustering (HC). The same scenario designs as shown in Figure \ref{fig:sim_design} are used, with the degree of phase noise--the concentration parameter $\kappa$ in a von Mises distribution--and observational noise (SNR) varied. As can be clearly seen, TROUT outperforms all methods in the high-phase noise (low $\kappa$) setting: this outperformance is most pronounced in Scenario 3 (bottom row) where the signals differ only in the relative phase of components. Time-domain clustering with oracle alignment performs as well as frequency-domain clustering, as expected.}
\vspace{-0.12in}
\label{fig:sim_perf}
\end{figure*}

\section{Computational Considerations}
\label{sec:computation}
The TROUT distance \eqref{eqn:trout} is defined by a one-dimensional optimization problem. This problem is non-convex in $\theta$, due to the sinusoidal behavior of the complex exponential, but simple enough that bisection or grid-search methods can be used to identify the minimum. Still, the non-convex problem may pose difficulties for composite optimization schemes \citep{Duchi:2018}. Instead, we recast $d_{\text{TROUT}}$ as a complex-valued optimization problem
\begin{equation}
  d_{\text{TROUT}}(\bu, \bx) = \min_{z \in \B_{\C}(r=1)} \|\bu - z\, \bx\|_2, \label{eqn:trout_cmplx}
\end{equation}
where $\B_{\C}(r=1)$ is the complex unit sphere. This problem is essentially \textit{complex least squares} with a norm constraint and has closed-form minimizer $z = \angle[\bx^H\bu] = \bx^H\bu / \|\bx\|\|\bu\|$. In multivariate settings, $z$ is optimized over unitary transforms, justifying the TROUT acronym.

Now that we can calculate $d_{\text{TROUT}}$ efficiently, we can develop algorithms for the TROUT clustering problem \eqref{eqn:trout_clust}. Following \citet{Chi:2015} and \citet{Weylandt:2020}, we use an alternating direction method of multipliers approach, whose updates are given by
\begin{align*}
  \bU^{(k+1)} &= \argmin_{\bU \in \C^{n \times p}} \frac{1}{2} d_{\text{TROUT}}(\bU, \bX)^2 + \\ &\phantom{\argmin}\quad \frac{\rho}{2}\left\|\bD\bU - \bV^{(k)} + \bZ^{(k)}\right\|_F^2 \\
  \bV^{(k+1)} &= \prox_{\lambda / \rho P(\cdot; \bw, q)}(\bD\bU^{(k+1)} + \bZ^{(k)}) \\
  \bZ^{(k+1)} &= \bZ^{(k)} + \bD\bU^{(k+1)} - \bV^{(k+1)}
\end{align*}
where $\bD$ is a suitable difference matrix and $P(\cdot; \bw, q)$ is the weighted convex fusion penalty function. For $q = 1, 2$, the proximal operator yields a simple closed-form update for $\bV$, while for $q = \infty$, the efficient algorithm of \citet{Condat:2016} can be used. 

Unlike standard convex clustering, the primal ($\bU$) update does not have a closed-form update and instead requires an iterative solver. Note that, for two vectors $\bu$ and $\bx$, $d_{\text{TROUT}}(\bu, \bx)$ can be equivalently expressed using the projection of $\bu$ onto the set of all ``phase-adjusted'' $\bx$ vectors: that is, 
\[d_{\text{TROUT}}(\bu, \bx) = \|\bu - P_{\tilde{\bx}}(\bu)\|_2 \text{ where } \tilde{\bx} = \{e^{\im \theta}\bx: \theta \in [0, 2\pi)\}\]
and where $P_{\tilde{\bx}}(\cdot)$ denotes the projection onto the set $\tilde{\bx}$. This sort of ``distance-to-set'' function was popularized in machine learning by \citet{Xu:2017-NIPS} and is well-behaved (single-valued and differentiable) for almost all $\bu, \bx$ \citep[Sections 18.6 and 24.3]{Bauschke:2017}. With this formulation, we can use a majorization-minimization (MM) approach \citep{Lange:2016} for the $\bU$ subproblem. In particular, note that
\[\frac{1}{2}d_{\text{TROUT}}(\bU, \bX)^2 \leq \frac{1}{2}\|\bU - P_{\mathcal{U}^{\{t\}}}(\bX)\|_F^2\]
where $\mathcal{U}^{\{t\}}$ is the ``phase-adjusted'' equivalence class of the previous iterate $\bU^{\{t\}}$ and $P_{\mathcal{U}^{\{t\}}}(\cdot)$ denotes projections thereon. This majorizes the TROUT distance because $P_{\mathcal{U}^{\{t\}}}(\bX) \in \mathcal{U}^{\{t\}}$ by construction and the distance to a set is always less than or equal to the distance to a fixed point within that set. This yields the inner iterates: 
\begin{align*}
  \tilde{\bU}^{\{t+1\}} &= \argmin_{\bU \in \C^{n \times p}} \frac{1}{2}\left\|\bU - P_{\mathcal{U}^{\{t\}}}(\bX)\right\|_F^2 + \\&\phantom{\argmin}\quad \frac{\rho}{2}\left\|\bD\bU - \bV^{(k)} + \bZ^{(k)}\right\|_F^2
\end{align*}
which have a closed-form solution 
\[\tilde{\bU}^{\{t+1\}} = (\bI + \rho\bD^T\bD)^{-1}(P_{\mathcal{U}^{\{t\}}}(\bX) + \rho \bD^T(\bV^{(k)} - \bZ^{(k)})).\]
Though non-convex, the $d_{\text{TROUT}}(\cdot, \bX)$ function is well-behaved and global convergence to a stationary point follows from Scenario 1(a) in Theorem 1 of \citet{Wang:2019}.

Additional performance improvements can be obtained by taking only a single MM step within each ADMM iteration \citep{Chen:2019-InexactADMM}, though this is typically unnecessary as the inner iterates tend to converge quite quickly. If the $\mathcal{O}(n^3)$ Cholesky factorization of $\bI + \rho \bD^T\bD$ is cached, the per iterate cost of this approach is $\mathcal{O}(n^2p)$, dominated by two rounds of triangular back-substitution, which scales efficiently with high sampling rates. For the simulations considered in the next section ($\bX \in \C^{60 \times 64}$), this algorithm solves Problem \eqref{eqn:trout_clust} to numerical precision in an average of 0.0825 seconds per value of $\lambda$ in a warm start scheme on a low-end desktop with 16 GB of RAM and an Intel i7-9700 processor running at 3.0 GHz. When combined with the ``one-step'' algorithmic regularization approach of \citet{Weylandt:2020}, performance is increased by an order of magnitude (0.000665 seconds per $\lambda$) without any noticeable drop in clustering accuracy.

\section{Simulation Study}
\label{sec:sims}

In this section, we compare TROUT clustering with various time- and frequency-domain clustering methods. We make repeated use of three simulation designs, shown in the left column of Figure \ref{fig:sim_design}. These three families  highlight  different challenges in time series clustering. The first scenario (top row) includes both a square wave and a (smoothed) triangle wave, which have slowly decaying spectra and exhibit Gibbs phenomena (rapid variation overshooting the target function) at their discontinuities; because of the slow spectrum decay, these signals have more relevant features than observations in each cluster, a situation which is well-known to challenge many clustering methods, while the discontinuities in the time-domain challenge warping or smoothing approaches. The second scenario (middle row) consists of sinusoids of different frequencies: these signals are easily separated in the frequency domain, but time-domain approaches struggle to identify structure. The third scenario (bottom row) is the most challenging: all three signals consist of pairs of sinusoids, differing only in relative phase. Power-only approaches are unable to distinguish these signals, and phase-noise poses a particular challenge for both time- and frequency-domain methods. 

From left-to-right, Figure \ref{fig:sim_design} shows the true cluster centroids; typical observations generated with a signal-to-noise ratio of 3 and with phase noise drawn from a mean zero von Mises random variable with concentration $\kappa = 10$; centroids estimated by dynamic time warping followed by hierarchical clustering (HC); centroids estimated by HC in the time-domain; centroids estimated by HC in the frequency-domain; and centroids estimated by TROUT. In each setting, signals are observed at 128 Hz, with 20 samples generated from each cluster for a total of 60 observations, yielding a spectral representation $\bX_C \in \C^{60 \times 64}$ or a time-domain representation $\bX_R \in \R^{60 \times 128}$. For legibility, a $5$-nearest neighbor smoother is super-imposed. While both TROUT and frequency-domain HC accurately separate clusters, only TROUT is able to correctly recover the shape of the true centroids (up to phase shift). 

Figure \ref{fig:sim_perf} uses the same scenarios, but varies the signal to noise ratio and the degree of phase noise. The sampling rate and number of observations is fixed as before. The adjusted Rand index used to measure clustering performance, is averaged over 50 replicates, with the oracle number of clusters made available to all methods. As can be seen, TROUT and frequency-domain HC are consistently among the best performing methods in all regimes, with the advantages of TROUT increasing as the phase noise increases ($\kappa$ decreases). For $\kappa = \infty$, corresponding to the case of no phase noise, the additional flexibility of TROUT registration slightly decreases performance. Focusing on the third scenario (bottom row) wherein signals vary only by their phase, we note that frequency domain HC is unable to distinguish signals in the high-phase noise setting, even with a high SNR. Time-domain methods including na\"ive HC and DTW are unable to correctly discover true clusters in this scenario at any SNR.

\section{Discussion and Extensions}
We have proposed a new approach to time series clustering which combines a convex fusion penalty with automatic registration. We have derived a closed-form solution for the optimal spectral-domain alignment and have used it to develop an efficient MM algorithm which attains a local optimum provably and efficiently. Simulation studies confirm the efficiency and stability of our approach across several regimes. Our method is particularly well-suited for situations of high ``phase noise,'' as would be expected when clustering naturally occurring stationary signals, though it also performs competitively in more standard domains. In the simulations shown in Section \ref{sec:sims}, we only considered basic spectral analysis (whole-sample DFT without tapering or smoothing), though our method is readily applied within more robust analytical pipelines.

The heart of our proposed method is the automatic and optimal alignment induced by the TROUT distance. While we have focused on clustering stationary univariate signals, the TROUT distance can be used anywhere a standard $\ell_2$ or Frobenius distance is used in machine learning. The basic framework of TROUT-based clustering can also be extended in several useful directions, including automatic amplitude adjustment, frequency isolation / band identification, and multivariate time series clustering. We conclude with a brief discussion of each of these extensions. 

As defined above, the TROUT distance \eqref{eqn:trout} only allows for phase adjustment. For signals which may be observed on different scales, where only intra-signal relative magnitude is meaningful, the complex form of the TROUT distance \eqref{eqn:trout_cmplx} can be relaxed to allow arbitrary complex $z$. As with standard TROUT, the relaxed problem has closed-form minimizer $z = \bx^H\bu / \|\bx\|^2$ and the associated clustering problem can be analyzed with a MM algorithm. Notably, this relaxed form has better convexity properties than the method discussed above. In certain problems, it may be useful to isolate certain frequencies which are most important to the clustering solution; this can be achieved by adding an $\ell_1$ or mixed $\ell_2/\ell_1$ penalty to the columns of $\bU$, as explored by \citet{Wang:2019-GECCO}. If, instead, it is useful to identify bands of frequencies which behave similarly, a TROUT-based form of convex bi-clustering \citep{Chi:2017} could be used to simultaneously cluster both observations and frequencies into clusters and meaningful bands respectively. Finally, if we instead seek to cluster $p$-variate time series, our data $\bX$ would now be an $n \times p \times p$ complex higher-order array (tensor) which we would seek to cluster along the first mode: in this case, the TROUT problem now consists of alignment by an arbitrary unitary transform. This problem can be solved via a unitary Procrustes problem, which has a closed-form solution in terms of the SVD of $\bX^H\bU$ and can be easily embedded in efficient tensor clustering algorithms \citep{Weylandt:2019b}. We plan to explore several of these extensions in future work. 

\section{References}
{\ninept \printbibliography[heading=none]}
\end{refsection}

\begin{refsection}
\onecolumn
\appendix
\setcounter{algorithm}{0}
\renewcommand{\thealgorithm}{A\arabic{algorithm}}

\section*{Supplementary Materials}
\section{Proofs} \label{app:proofs}
As discussed in Section \ref{sec:computation}, the TROUT distance between two vectors can be computed by solving the complex-valued (Wirtinger) optimization problem
\[d_{\text{TROUT}}(\bu, \bx) = \min_{z \in \B_{\C}(r = 1)} \|\bu - z \bx\|_2\]
If we consider the column vectors $\bu$ and $\bx$ as $p \times 1$ matrices and $z$ as a $1 \times 1$ matrix, we note that this is a special case of the \emph{Unitary} Procrustes problem: 
\[\argmin_{\bZ \in \mathfrak{U}_k} \|\bU - \bX\bZ \|_F^2\]
where $\mathfrak{U}_k$ denotes the unitary group of dimension $k$; that is, the set of $k \times k$ matrices satisfying $\bZ^H\bZ = \bZ\bZ^H = \bI_{k \times k}$. The unitary Procrustes problem has a closed-form solution given by 
\[\bA\bB^H \text{ where } \bA\bD\bB^H = \textsf{SVD}(\bX^H\bU)\]
(The Procrustes problem is typically presented with $\bZ\bX$ instead of $\bX\bZ$: in that case, the solution uses the SVD of $\bU\bX^H$.)
The unitary Procrustes problem does not often appear in the literature \emph{per se}, but analysis of the orthogonal (real-valued) Procrustes problem translates directly to the complex-domain \citep{Schonemann:1966,Elden:1999}. With this in hand, we are able to give a closed-form solution for the TROUT minimizer. 

\begin{proposition} Let $\bu, \bx \in \C^p$ be arbitrary complex vectors. Then 
\[\angle[\bx^H\bu] = \frac{\bx^H\bu}{\|\bx\|\|\bu\|} = \argmin_{z \in \B_{\C}(r = 1)} \|\bu - z \bx\|_2\]
and
\[d_{\text{TROUT}}(\bu, \bx) = \min_{\theta \in [0, 2\pi)} \|\bu e^{\im \theta}\bx\| = \left\|\bu - \angle[\bx^H\bu] \bx\right\| \]
\end{proposition}
\begin{proof}
Using the closed-form solution for the Procrustes problem given above, it suffices to calculate the SVD of $\bx^h\bu$. Note that we have the analytic SVD
\[\bx^H\bu = \begin{pmatrix}\frac{\bx}{\|\bx\|}\end{pmatrix}\begin{pmatrix}\|\bu\|\bx\|\end{pmatrix}\begin{pmatrix} \frac{\bu}{\|\bu\|}\end{pmatrix}^H\]
so the Procrustes solution is given by 
\[\begin{pmatrix}\frac{\bx}{\|\bx\|}\end{pmatrix}\begin{pmatrix} \frac{\bu}{\|\bu\|}\end{pmatrix}^H = \angle[\bx^H\bu]\]
as claimed. The second claim follows by substitution in the definition of $d_{\text{TROUT}}$.
\end{proof}
\section{Additional References}
\printbibliography[heading=none]

\end{refsection}
\end{document}